\documentclass[11pt, twoside, letter]{article}
\usepackage[margin = 1in]{geometry}
\usepackage[utf8]{inputenc}
\usepackage{amsmath, amssymb, amsfonts, cite, amsopn, amssymb}
\usepackage{ifthen}
\usepackage[vlined,ruled,linesnumbered]{algorithm2e}
\usepackage{comment}

\usepackage{amsthm}
\usepackage{hyperref}

\usepackage{verbatim}
\usepackage{bbm}
\usepackage{color}
\usepackage{graphicx, epstopdf}
\usepackage{mathtools}
\usepackage{float}

\newtheorem{theorem}{Theorem}
\newtheorem{corollary}{Corollary}
\newtheorem{lemma}{Lemma}

\newtheorem{proposition}{Proposition}
\newtheorem{claim}{Claim}

\newtheorem{assumption}{Assumption}
\newtheorem{remark}{Remark}
\newtheorem{definition}{Definition}
\newtheorem{example}{Example}

\newcommand{\BALD}{\begin{aligned}}
\newcommand{\EALD}{\end{aligned}}
\newcommand{\BALDS}{\begin{aligned*}}
\newcommand{\EALDS}{\end{aligned*}}
\newcommand{\BCAS}{\begin{cases}}
\newcommand{\ECAS}{\end{cases}}
\newcommand{\BEAS}{\begin{eqnarray*}}
\newcommand{\EEAS}{\end{eqnarray*}}
\newcommand{\BEQ}{\begin{equation}}
\newcommand{\EEQ}{\end{equation}}
\newcommand{\BIT}{\begin{itemize}}
\newcommand{\EIT}{\end{itemize}}
\newcommand{\BMAT}{\begin{bmatrix}}
\newcommand{\EMAT}{\end{bmatrix}}
\newcommand{\BNUM}{\begin{enumerate}}
\newcommand{\ENUM}{\end{enumerate}}
\newcommand{\eg}{{\it e.g.}}

\newcommand{\BA}{\begin{array}}
\newcommand{\EA}{\end{array}}

\newcommand{\reals}{\mathbf{R}}

\newcommand{\diag}{\mathop{\mathbf{diag}}}

\newcommand{\nullspace}{\mathcal{N}}

\newcommand{\range}{\mathcal{R}}

\DeclareMathOperator{\linspan}{span}
\newcommand{\norm}[1]{\left\| #1 \right\|}

\newcommand{\rmD}{\mathrm{D}}

\newcommand{\Id}{\textup{Id}}

\newcommand{\cA}{\mathcal{A}}

\newcommand{\cC}{\mathcal{C}}
\newcommand{\cD}{\mathcal{D}}

\newcommand{\cK}{\mathcal{K}}
\newcommand{\cM}{\mathcal{M}}
\newcommand{\cN}{\mathcal{N}}

\newcommand{\cR}{\mathcal{R}}

\newcommand{\cT}{\mathcal{T}}

\newcommand{\cX}{\mathcal{X}}

\newcommand{\R}{\mathbb{R}}

\def\bar{\overline}

\begin{document}

\title{First-order Methods Almost Always Avoid Saddle Points \thanks{This paper significantly extends upon the special case of gradient descent dynamics developed in the conference proceedings of the authors \cite{pmlr-v49-lee16,panageas2017gradient}.}}

\author{Jason D. Lee\\USC\\jasonlee@marshall.usc.edu
\and Ioannis Panageas\\MIT\\ioannis@csail.mit.edu
\and Georgios Piliouras\\SUTD\\georgios@sutd.edu.sg
\and Max Simchowitz\\UC Berkeley\\msimchow@berkeley.edu
\and Michael I. Jordan\\UC Berkeley\\jordan@cs.berkeley.edu
\and Benjamin Recht\\UC Berkeley\\brecht@berkeley.edu
}

\date{}
\maketitle

\begin{abstract}
We establish that first-order methods avoid saddle points
 for almost all initializations. Our results apply to a wide variety of first-order methods, including gradient descent, block coordinate descent, mirror descent and variants thereof. The connecting thread is that such algorithms can be studied from a dynamical systems perspective in which appropriate instantiations of the Stable Manifold Theorem allow for a global stability analysis. Thus, neither access to second-order derivative information nor randomness beyond initialization is necessary to provably avoid saddle points.

\end{abstract}

\section{Introduction}
Saddle points have long been regarded as a major obstacle for non-convex optimization over continuous spaces. It is well understood that in many applications of interest, the number of saddle points significantly outnumber the number of local minima, which is especially problematic when the solutions associated with worst-case saddle points are considerably worse than those associated with worst-case local minima 
 \cite{dauphin2014identifying,pascanu2014saddle,choromanska2014loss}.
Moreover, it is not hard to construct examples where a worst-case initialization of gradient descent (or other  first-order methods) provably converge to saddle points \cite[Section 1.2.3]{nesterov2004introductory}.

The main message of our paper is that, under very mild regularity conditions, saddle points have little effect on the asymptotic behavior of first-order methods.  Building on tools from the theory of dynamical systems, we generalize  recent analysis of gradient descent  \cite{pmlr-v49-lee16,panageas2017gradient} to establish that a wide variety of first-order methods --- including gradient descent, proximal point algorithm, block coordinate descent, mirror descent --- avoid so-called ``strict'' saddle points for almost all initializations; that is, saddle points where the Hessian of the objective function admits at least one direction of negative curvature (see Definition~\ref{def:strict_saddle}).

Our results provide a unified theoretical framework for analyzing the asymptotic behavior of a wide variety of classic optimization heuristics in non-convex optimization. Furthermore, we believe that furthering our understanding of the behavior and geometry of deterministic optimization techniques with random initialization can serve in the development of stochastic algorithms which improve upon their deterministic counterparts and achieve strong convergence-rate results; indeed, such insights have already led to significant improves in modifying gradient descent to navigate saddle-point geometry \cite{du2017gradient,jin2017escape}.

\subsection{Related work}
In recent years, the optimization and machine learning communities have dedicated much effort to understanding the geometry of non-convex landscapes by searching for unified geometric properties which could be leverage by general-purpose optimization techniques. The strict saddle property (Definition~\ref{def:strict_saddle}) is one such property which has been shown to hold in a wide and diverse range of salient objective functions: PCA, a fourth-order tensor factorization \cite{ge2015escaping}, formulations of dictionary learning  \cite{sun2015complete2,sun2015complete1}, phase retrieval \cite{sun2016phase}, low-rank matrix factorizations \cite{ge2016matrix,ge2017no,bhojanapalli2016global}, and simple neural networks \cite{soltanolkotabi2017theoretical,du2017convolutional,brutzkus2017globally}. It is also known that, in the worst case, the strict saddle property is unavoidable as finding descent-directions at critical points with degenerate Hessians is NP-hard in general \cite{murty1987some}.

Earlier work had shown that first-order descent methods can circumvent strict saddle points, provided that they are augmented with unbiased noise whose variance is sufficiently large in each direction. For example, \cite{pemantle1990nonconvergence} establishes convergence of the Robbins-Monro stochastic approximation to local minimizers for strict saddle functions. More recently, \cite{ge2015escaping} give quantitative rates on the convergence of noisy gradient descent  to local minimizers, for strict saddle functions.

To obtain provable guarantees without the addition of stochastic noise, \cite{sun2015complete2,sun2015complete1} and \cite{sun2016phase} adopt trust-region methods which leverage Hessian information in order to circumvent saddle points. This approach represents a refinement of a long tradition of related, ``second-order'' strategies, including: a modified Newton's method with curvilinear line search \cite{more1979use}, the modified Cholesky method \cite{gill1974newton}, trust-region methods \cite{conn2000trust}, and the related cubic regularized Newton's method \cite{nesterov2006cubic}, to name a few. Specialized to deep learning applications, \cite{dauphin2014identifying,pascanu2014saddle} have introduced a saddle-free Newton method.

However, such curvature-based optimization algorithms have a per-iteration computational complexity which scales quadratically or even cubically in the dimension $d$, rendering them unsuitable for optimization of high-dimensional functions. In more recent work, several works have presented faster curvature-based methods including \cite{royer2017complexity, liu2017noisy,reddi2017generic} by combining fast first-order methods with fast eigenvector algorithms, to obtain lower per-iteration complexity.

Fortunately, it appears that neither the addition of isotropic noise, nor the use of second-order methods are  necessary for circumventing saddle points. For example, recent work by  \cite{jin2017escape} showed that by carefully perturbing the iterates of gradient descent in the vicinity of possible saddles results in a first-order method which converges to local minimizers in a number of iterations with only poly-logarithmic dimension dependence. Moreover, many recent works have shown that, even without any random perturbations, a combination of gradient descent and a smart-initialization provably converges to the global minimum for a variety of non-convex problems: such settings include matrix factorization \cite{Keshavan09,zhaononconvex} , phase retrieval \cite{candes2015phase,cai2015optimal}, dictionary learning \cite{arora2015simple}, and latent-variable models \cite{zhang2014spectral,belkin2014basis}. While our results only guarantee convergence to local minimizers, they eschew the need for complex and often computationally prohibitive initialization procedures.

In addition to what has been established theoretically, there is a broadly-accepted folklore in the field that running gradient descent with a random initialization is sufficient to identity a local optima. For example, the authors of \cite{sun2016phase} empirically observe gradient descent with $100$ random initializations on the phase retrieval problem always converges to a local minimizer, one whose quality matches that of the solution found using more costly trust-region techniques. It is the purpose of this work to place these intuitions on firm mathematical footing.

Finally, we emphasize that their are many settings in which all local optima (but not saddles!) have objective values which are nearly as small as those of the global minima; see for example \cite{ge2016matrix,ge2017no,soltanolkotabi2017theoretical,sun2015nonconvex,sun2015complete1}. Some preliminary results have suggested that this may be a a quite general phenomenon. For example, \cite{choromanska2014loss} study the loss surface of a particular Gaussian random field as a proxy for understanding the objective landscape of deep neural nets. The results leverage the Kac-Rice Theorem \cite{adler2009random,auffinger2013random}, and establish that critical points with more positive eigenvalues have lower expected function value, often close to that of the global minimizer.
We remark that functions drawn from this Gaussian random field model share the strict saddle property defined above, and so our results apply in this setting.  On the other hand, our results are considerably more general, as they do not place stringent generative assumptions on the objective function $f$.

\subsection{Organization}
The rest of the paper is organized as follows. Section \ref{sec:prelim} introduces the notation and definitions used throughout the paper. Section \ref{sec:intuition} provides an intuitive explanation for why it is unlikely that gradient descent converges to a saddle point, by studying a non-convex quadratic and emphasizing the analogy with power iteration. Section \ref{sec:smt-unstable-fp} develops the main technical theorem, which uses the stable manifold theorem to show that the stable set of unstable fixed points has measure zero. Section \ref{sec:algs} applies the main theorem to show that gradient descent, block coordinate descent, proximal point, manifold gradient descent, and mirror descent all avoid saddle points. Finally, we conclude in Section \ref{sec:conclusion} by suggesting several directions of future work.

\section{Preliminaries}
\label{sec:prelim}
Throughout the paper, we will use $f: \cX \to \reals $ to denote a real-valued function in $C^2$, the space of twice-continuously differentiable functions.

\begin{definition}[Strict Saddle]\label{def:strict_saddle}
When $\cX= \reals^d$,
\begin{enumerate}
\item A point $x^*$ is a critical point of $f$ if $\nabla f(x^*)=0$.

\item A point $x^*$ is a strict saddle point\footnote{For the purposes of this paper, strict saddle points include local maximizers.} of $f$ if $x^*$ is a critical point and  $\lambda_{\min} (\nabla^2 f(x^*)) <0$. Let $\cX^*$ denote the set of strict saddle points.
\end{enumerate}

When $\cX$ is a manifold, the same definition applies, but with gradient and Hessian replaced by the  Riemannian gradient $\nabla_{R} f(x)$ and Riemannian Hessian $\nabla^2_{R} f(x)$. See Section \ref{sec:manifold-gd} for details, and Chapter 5.5 of \cite{absil2010optimization}.

\end{definition}

Our interest is in the attraction region of an optimization algorithm $g$, viewed as a mapping from $\cX \to \cX$. The iterates of the algorithm are generated by the sequence
\begin{align*}
x_{k}=g(x_{k-1})=g^k (x_0),
\end{align*}
where $g^k$ is the $k$-fold composition of $g$.
As an example, gradient descent corresponds to $g(x) = x_k - \alpha \nabla f(x_{k})$.

Since we are interested in the region of attraction of a critical point, we provide the definition of the stable set.
\begin{definition}[Global Stable Set]
The global stable set of the strict saddles is the set of initial conditions where iteration of the mapping $g$ converges to a strict saddle. This is defined as
\begin{align*}
W_g = \{ x_0: \lim_k g^k (x_0) \in \cX^* \}.
\end{align*}
\end{definition}

\section{Intuition}

\label{sec:intuition}
To illustrate why gradient descent and related first-order methods do not converge to saddle points, consider the case of a non-convex quadratic, $f(x)=\frac{1}{2} x^T H x$.  Without loss of generality, assume $H = \diag(\lambda_1,...,\lambda_n)$  with $\lambda_1, ..., \lambda_k >0$ and $\lambda_{k+1}, \dots, \lambda_n <0$.  $x^*=0$ is the unique critical point of this function and the Hessian at $x^*$ is $H$. Gradient descent initialized from $x_0$ has iterates
\[
	x_{t+1} =g(x_t)=\sum_{i=1}^n (1-\alpha \lambda_i)^{t+1} \langle e_i,x_0 \rangle e_i\,.
\]
where $e_i$ denote the standard basis vectors. This iteration resembles power iteration with the matrix $I - \alpha H$.

Let $L = \max |\lambda_i|$, and suppose $\alpha < 1/L$.  Thus we have $(1- \alpha \lambda_i)< 1$ for $i \le k$ and $(1-\alpha \lambda_i) >1$ for $i >k$. If  $x_0 \in E_s:=\linspan(e_1,\ldots, e_k)$, then $x_t$ converges to the saddle point at zero since $(1-\alpha \lambda_i)^{t+1} \to 0$. However, if $x_0$ has a component outside $E_s$ then gradient descent diverges to $\infty$.  For this simple quadratic function, we see that the global stable set (attractive set) of zero is the subspace $E_s$.  Now, if we choose our initial point at random, the probability of that point landing in $E_s$ is zero as long as $k<n$ (i.e., $E_s$ is not full dimensional).

As an example of this phenomenon for a non-quadratic function, consider the following example from \cite[Section 1.2.3]{nesterov2004introductory}. Letting $f(x,y) = \frac12 x^2 +\frac14 y^4-\frac12 y^2$, the corresponding gradient mapping is
\begin{align*}
g(x) &= \begin{bmatrix}
(1-\alpha) x\\
(1+\alpha) y - \alpha y^3
\end{bmatrix}.
\end{align*}
The critical points are
\begin{align*}
z_1 = \begin{bmatrix} 0 \\0 \end{bmatrix}, \quad z_2 = \begin{bmatrix} 0 \\ -1 \end{bmatrix},\quad
z_3 = \begin{bmatrix} 0 \\ 1 \end{bmatrix}.
\end{align*}
 The points $z_2$ and $z_3$ are isolated local minima, and $z_1$ is a saddle point.

Gradient descent initialized from any point of the form $\begin{bmatrix} x \\ 0 \end{bmatrix}$ converges to the saddle point $z_1$. Any other initial point either diverges, or converges to a local minimum, so the stable set of $z_1$ is the $x$-axis, which is a zero-measure set in $\reals^2$. By computing the Hessian,
\begin{align*}
\nabla ^2 f(x) = \begin{bmatrix}
1 & 0 \\
0 & 3y^2 -1
\end{bmatrix},
\end{align*}
we find that $\nabla ^2 f(z_1)$ has one positive eigenvalue with eigenvector that spans the $x$-axis, thus agreeing with our above characterization of the stable set. If the initial point is chosen randomly, there is zero probability of initializing on the $x$-axis and thus zero probability of converging to the saddle point $z_1$.

For gradient descent, the local attractive set of a critical point $x^*$ is well-approximated by the span of the eigenvectors corresponding to positive eigenvalues of the Hessian.  By an application of Taylor's theorem, one can see that if the initial point $x_0$ is uniformly random in a small neighborhood around $x^*$, then the probability of initializing in the span of these eigenvectors is zero whenever there is a negative eigenvalue.  Thus, gradient descent initialized at $x_0$ will leave the neighborhood of $x^*$. Although this argument provides valuable intuition, there are several difficulties with formalizing this argument: 1)  $x_0$ is randomly distributed over the entire domain, not a small neighborhood around $x^*$, and Taylor's theorem does not provide any global guarantees, and 2) it does not rule out converging to a different saddle point.

\section{Stable Manifold Theorem and Unstable Fixed Points}
\label{sec:smt-unstable-fp}
\subsection{Setup}
For the rest of this paper, $g$ is a mapping from $\cX$ to itself, and $\cX$ is a $d$-dimensional manifold without boundary. Recall that a $\mathcal{C}^k$-smooth, $d$-dimensional manifold is a space $\cX$, together with a collection of \emph{charts} $\{(U_{\alpha},\phi_{\alpha})\}$, called an \emph{atlas}, where each $\phi_{\alpha}$ is a homeomorphism from an open subset $U_{\alpha} \subset \cX$ to $\R^d$. The charts are required to be compatible in the sense that, whenever $U_{\alpha} \cap U_{\beta} \ne \emptyset$, then the transition map $\phi_{\alpha} \circ \phi_{\beta}^{-1}$ is a $\mathcal{C}^k$ map from $\phi_{\beta}(U_{\beta} \cap U_{\alpha}) \to \R^d$. We also require that $\bigcup_{\alpha} U_{\alpha} = \cX$, and $\cX$ is second countable, which means that for any set $U$ contained in $\bigcup_{\alpha \in \mathcal{I}}U_{\alpha}$ for some index set $\mathcal{I}$, there exists a countable set $\mathcal{J} \subset \mathcal{I}$ such that $U \subset \bigcup_{\alpha \in \mathcal{J}}U_{\alpha}$. We can now recall the definition of a measure zero subset of a manifold:
\begin{definition}[Section 5.4 of \cite{mikusinski2012introduction}]
Given a $d$-dimensional manifold $\cX$, we say that a set $E \subset \cX$ is measure zero if there is an atlas $\{U_i,\phi_i\}_{i \ge 1}$ such that $\phi_i(E \cap U_i)$ has Lebesgue-measure zero as a subset of $\R^d$. In this case, we use the shorthand $\mu(E) = 0$. The measure zero property is independent of the choice of atlas \cite[Chapter 5]{mikusinski2012introduction}.
\end{definition}

\begin{definition}[Chapter 3 of \cite{absil2010optimization}]
The differential of the mapping $g$, denoted as $\rmD g (x)$, is a linear operator from $\cT(x)  \to \cT(g(x))$, where $\cT(x)$ is the tangent space of $\cX$ at point $x$. Given a curve $\gamma$ in $\cX$ with $\gamma(0) =x$ and $\frac{d\gamma}{dt}(0) =v \in \cT(x)$,  the linear operator is defined as $\rmD g (x)  v = \frac{d( g \circ \gamma) }{dt} (0) \in \cT(g(x))$. The determinant of the linear operator $\det(\rmD g(x)  ) $ is the determinant of the matrix representing $\rmD g(x)$ with respect to an arbitrary basis\footnote{The determinant is invariant under similarity transformations, so is independent of the choice of basis.}.
\end{definition}

\begin{lemma}
	Let $E \subset \cX$ be a measure zero subset. If $\det(Dg(x)) \neq 0$ for all $x \in \cX$ , then $\mu (g^{-1}(E))$ has measure zero.
	\label{lem:measure-0-local-diff}
\end{lemma}
\begin{proof}
	For clarity, let $h=g^{-1}$.
	Let $(V_i , \psi_i) $ be a countable collection of charts of the co-domain of $g$. By countable additivity of measure, it suffices to show that each $h(E) \cap V_i$ is measure zero. Without loss of generality, we may assume that $h(E)$ is contained in a chart $( V, \psi)$, else we could repeat the same argument for each element of the chart.
	
	We wish to show that $\mu( \psi \circ h  (E)) =0$. Let $(U_j, \phi_j)$ be another countable collection of charts of the domain of $g$. Define $E_j = E \cap U_j $, and note that $E= \cup_{i=1}^{\infty} E_i  = \cup \phi^{-1} \circ \phi\ (E_i)$. Thus
	\begin{align*}
	\mu(\psi \circ h (E)) &= \mu(\psi \circ h (\cup_{i} \phi^{-1} \circ \phi (E_i)))\\
	&\le \sum_{i=1}^\infty  \mu ( \psi \circ h \circ \phi ^{-1}  (\phi(E_i))) .
	\end{align*}
	By assumption, $\phi(E_i)$ is measure zero. The function $ \psi \circ h \circ \phi^{-1} =\psi \circ g^{-1} \circ \phi^{-1}$ is $C^1$ if $\det( Dg) \neq 0$, and thus locally Lipschitz, so preserves measure zero sets. By countable additivity and the displayed equation above, $E$ has measure zero.
	
\end{proof}
\subsection{ Unstable Fixed Points}

\begin{definition}[Unstable fixed point]
Let $$\cA^* _g= \{x: g(x) =x, \max_{i} |\lambda_i (Dg(x))|>1 \}$$ be the set of fixed points where the differential has at least a single eigenvalue with magnitude greater than one. These are the unstable fixed points.
\end{definition}

\begin{theorem}[Theorem III.7, \cite{shub1987global}]
	Let $x^*$ be a fixed point for the $C^r$ local diffeomorphism $g: \cX \to \cX$. Suppose that $E = E_s \oplus E_u$, where $E_s$ is the span of the eigenvectors corresponding to eigenvalues of magnitude less than or equal to one of $D g(x^*)$, and $E_u$ is the span of the eigenvectors corresponding to eigenvalues of magnitude greater than one of $D \phi(0)$. Then there exists a $C^r$ embedded disk $W^{cs}_{loc}$ that is tangent to $E_s$ at $x^*$ called the \emph{local stable center manifold}.  Moreover, there exists a neighborhood $B$ of $x^*$, such that $g(W^{cs}_{loc} ) \cap B \subset W^{cs} _{loc}$, and $\cap_{k=0}^\infty g^{-k} (B) \subset W^{cs}_{loc}$.
	\label{thm:stable-center-manifold}
\end{theorem}

\begin{theorem}
Let $g$ be a $\cC^1$ mapping from $\cX \to \cX$ and $\det( \rmD g(x))\neq 0 $ for all $x \in \cX$. Then the set of initial points that converge to an unstable fixed point has measure zero,  $\mu(\{x_0: \lim x_k \in \cA^*_g \} ) =0$.
\label{thm:main-general}
\end{theorem}
\begin{proof}
For each $x^* \in \cA^* _g$, there is an associated open neighborhood $B_{x^*}$ promised by the Stable Manifold Theorem \ref{thm:stable-center-manifold}. $ \cup_{x^* \in \cA^*} B_{x^*}$ forms an open cover, and since $\cX$ is second-countable we can extract a countable subcover, so that $\cup_{x^* \in \cA^*} B_{x^*}= \cup_{i=1}^\infty B_{x^*_i}$.

Define $W=\{x_0: \lim_k  x_k \in \cA^*_g\}$. Fix a point $x_0 \in W$. Since $x_k \to x^* \in \cA^*_g $, then for some non-negative integer $T$ and all $t\ge T$,  $g^t(x_0) \in \cup_{x^* \in \cA^*} B_{x^*}$. Since we have a countable sub-cover, $g^t(x_0) \in B_{x^*_i} $ for some  $x^* _i \in \cA^*$ and all $t\ge T$. This implies that $g^t (x_0) \in \cap_{k=0}^\infty\ g^{-k} ( B_{x^*_i})$ for all $t\ge T$. By Theorem \ref{thm:stable-center-manifold}, $ S_i \triangleq \cap_{k=0}^\infty g^{-k} ( B_{x^*_i})$  is a subset of the local center stable manifold  which has co-dimension at least one, and $S_i$ is thus measure zero.

Finally, $g^T(x_0) \in S_i$ implies that $x_0 \in g^{-T} ( S_i)$. Since $T$ is unknown we union over all non-negative integers, to obtain $x_0 \in \cup_{j=0}^\infty g^{-j} (S_i)$. Since $x_0$ was arbitrary, we have shown that $W \subset \cup_{i=1}^\infty  \cup_{j=0}^\infty g^{-j} (S_i)$. Using Lemma \ref{lem:measure-0-local-diff} and that countable union of measure zero sets is measure zero, $W$ has measure zero.
\end{proof}

Next, we state a simple corollary that only requires verifying $\det( \rmD g(x) ) \neq 0$, and $\cX^* \subset \cA^*_g$.
\begin{corollary}
	\label{cor:main-corr}
	Under the same conditions as Theorem \ref{thm:main-general}, and in addition assume $\cX^* \subset \cA^* _g$, then $\mu(W_g ) =0$.
	\end{corollary}
\begin{proof}
	Since $\cX^* \subset \cA^* _g$, then $W_g \subset \{x_0: \lim_k g^k (x_0) \in \cA^*_g \}$. Using Theorem \ref{thm:main-general}, $\mu(W_g) =0$.
\end{proof}

\section{Application to Optimization}
\label{sec:algs}
\subsection{Gradient Descent and Proximal Point}
As an application of Theorem \ref{thm:main-general}, we show that gradient descent avoids saddle points. Consider the gradient descent algorithm with step-size $\alpha$:
\begin{align}
x_{k+1}=g(x_k) \triangleq x_k - \alpha \nabla f(x_k). \label{eq:g-gd}
\end{align}
\begin{assumption}[Lipschitz Gradient]
	\label{assumption:lipschitz-gradient-descent}
	Let $f \in \cC^2$, and  $\norm{\nabla^2 f(x)}_2 \le L$.
\end{assumption}

\begin{proposition}
	Every strict saddle point $x^* $ is an unstable fixed point of gradient descent, meaning $\cX^* \subset \cA^* _g$ .
	\label{prop:gd-unstable}
	\end{proposition}
\begin{proof}
	First we verify that critical points of $f$ are fixed points of $g$. Since $\nabla f(x) =0$, then $g(x) = x-\alpha \nabla f(x) =x$ and is a fixed point.
	
 At a strict saddle $ x^* \in \cX^*$, $\rmD g(x^*) =\Id-\alpha \nabla ^2 f(x^*)$ with eigenvalues $ 1- \alpha \lambda_i$ , where $\lambda_i $ are eigenvalues of $\nabla^2 f(x^*)$. Since $x^*$ is a strict saddle, then there is at least one eigenvalue $\lambda<0$, and $1-\alpha \lambda_i >1$.  Thus $x^* \in \cA^* _g$.
\end{proof}

\begin{proposition}
	Under Assumption \ref{assumption:lipschitz-gradient-descent} and $\alpha< \frac1L$, then $\det(\rmD g (x)) \neq 0$.
	\label{prop:gd-det-not-0}
	\end{proposition}
\begin{proof}
By a straightforward calculation
\begin{align*}
\rmD g(x) = \Id- \alpha \nabla ^2 f(x)= \Id - \alpha V D V^T,
\end{align*}
where $\nabla^2 f(x) = V DV^T$. The eigenvalues of $Dg(x)$ are $1- \alpha \lambda_i $, and so
\begin{align*}
\det(\rmD g(x) ) =\prod_{i} ( 1-\alpha \lambda_i).
\end{align*}
Using the Lipschitz gradient assumption, $\alpha< 1/|\lambda_i |$ and each term in the product is positive, so $\det(\rmD g(x)) >0$.

\end{proof}

\begin{corollary}
	Let $g$ be the gradient descent algorithm as defined in Equation \eqref{eq:g-gd}. Under Assumption \ref{assumption:lipschitz-gradient-descent} and $\alpha< \frac1L$,  the stable set of the strict saddle points has measure zero, meaning $\mu(W_g)=0$.
\end{corollary}
\begin{proof}
	The proof is a straightforward application of the previous two Propositions and Corollary \ref{cor:main-corr}. Proposition \ref{prop:gd-unstable} shows that $\cX^* \subset \cA^* _g$, and Proposition \ref{prop:gd-det-not-0} shows that $\det( \rmD g (x)) \neq 0$. By applying Corollary \ref{cor:main-corr}, we conclude that $\mu( \{x_0: \lim_k g^k (x_0) \in \cX^*\}) =0$.
\end{proof}

\subsection{Proximal Point}
The proximal point algorithm is given by the iteration
\begin{align}
x_{k+1} = g(x) \triangleq \arg\min_{z} f(z) +\frac{1}{2\alpha} \norm{x_k -z}^2.
\label{eq:g-proxpoint}
\end{align}

\begin{proposition}
		Under Assumption \ref{assumption:lipschitz-gradient-descent} and $\alpha< \frac1L$, then
		\begin{enumerate}
			\item  $\det( \rmD g(x)) \neq 0$.
			\item Every strict stable point $x^*$ is an unstable fixed point of proximal point, meaning $\cX^* \subset \cA^* _g$.
		\end{enumerate}
	\label{prop:proxpoint}
	\end{proposition}

\begin{proof}
	Since $\nabla f$ is $L$-Lipschitz, $f(z) +\frac{1}{2\alpha} \norm{x-z}^2$ is strongly convex for $\alpha < \frac1L$,  and the $\arg\min$ is well-defined and unique. By the optimality conditions, $g(x) +\alpha \nabla f(g(x)) =x$. By implicit differentiation, $
	\rmD g(x) +\alpha \nabla^2 f(g(x)) \rmD g(x) = \Id$, and so
	\begin{align*}
	\rmD g(x) = (\Id+\alpha \nabla^2 f(g(x)))^{-1} .
	\end{align*}
	At a strict saddle $x^*$, $\rmD g(x^*) = (\Id+\alpha \nabla^2 f(x^*) )^{-1}$, and thus has an eigenvalue greater than one. For $\alpha <\frac1L$, $\rmD g (x)$ is invertible, and thus $\det( \rmD g (x)) \neq 0$.
	
\end{proof}

By combining Proposition \ref{prop:proxpoint} and Corollary \ref{cor:main-corr}, we have the following:
\begin{corollary}[Proximal Point]
	Let $g$ be the proximal point algorithm as defined in Equation \eqref{eq:g-proxpoint}. Under Assumption \ref{assumption:lipschitz-gradient-descent} and $\alpha< \frac1L$,  the stable set of the strict saddle points has measure zero, meaning $\mu(W_g)=0$.
\end{corollary}

\subsection{Coordinate Descent}

\begin{algorithm}[H]
		\caption{Coordinate Descent} \label{alg:cd}
	\textbf{Input:} Function $f: \mathbb{R}^d \to \R$, step size $\alpha$, initial point $x_0 \in \mathbb{R}^n$ \\
	\textbf{For} $k = 0,1,\dots$,\\
	\textbf{For index} $i = 1,\dots,n$\\
	\Indp $x^i_{k+1} \leftarrow x^i_{k} -\alpha \frac{\partial f(y^{i-1}_{k})}{\partial x_i}$, where
	\begin{eqnarray}
	y^0_k = x_k \textrm{ and }y^i_k = (x^1_{k+1},\ldots,x^i_{k+1},x^{i+1}_{k},\ldots,x^n_k)
	\end{eqnarray}
\end{algorithm}

We define $g_i(x)=x - \alpha (0,\ldots,0, \frac{\partial f(x)}{\partial x_i},0,\ldots,0)$ to be the coordinate descent update of index $i$ in Algorithm \ref{alg:cd}. One iteration of coordinate gradient descent corresponds to the update
 \begin{align}
x_{k+1} = g(x_k)= g_n \circ g_{n-1} \circ \ldots \circ g_1(x).
\label{eq:g-cd}
 \end{align}

 \begin{assumption}[Lipschitz Coordinate Gradient]
 	\label{assumption:lipschitz-cd}
 	Let $f \in \cC^2$, and  \\$\max_{i \in [d]} |e_i ^T\nabla^2 f(x) e_i| \le L_{\max}.$
 \end{assumption}

\begin{lemma}
The differential is
	\begin{equation}\label{eq:jacobian}
	\rmD g (x_k) = \prod_{j=1}^n (\Id - \alpha e_{n-j+1}e_{n-j+1}^{T} \nabla^2 f(y_k^{n-j})),
	\end{equation}
	where $e_i$ is a standard basis vector.
\end{lemma}
\begin{proof}
	This is an application of the chain rule. The differential of the composition of two functions $f \circ h$ is just $\rmD f(h(x)) \cdot \rmD h (x)$. By repeatedly applying this and observing that $\rmD g_i (x) = \Id - \alpha e_i e_i^{T} \nabla^2 f(x)$, we have the result.
\end{proof}

\begin{proposition} Under Assumption \ref{assumption:lipschitz-cd} and $\alpha< \frac{1}{L_{\max}}$, then $\det( \rmD g(x)) \neq 0$.
	\label{prop:cd-det}
\end{proposition}
\begin{proof}
	It suffices to prove that every term of Equation \ref{eq:jacobian} is an invertible matrix. Using the matrix determinant lemma, the characteristic polynomial of the matrix $\Id - \alpha e_i e_i^{T}\nabla^2 f(x)$ is equal to $(\lambda-1)^{n-1} (\lambda - 1+\alpha \frac{\partial ^2 f(x)}{\partial x_i^2})$. For $\alpha < \frac{1}{ \frac{\partial ^2 f(x)}{\partial x_i^2}}$, the eigenvalues of $\rmD g_i $ are all positive, and thus $\rmD g_i$ is invertible.
\end{proof}

\begin{proposition}[Instability at saddle points]\label{prop:cd-unstable}
	Every strict saddle point $x^*$ is an unstable fixed point of coordinate descent, meaning $\cX^* \subset \cA^* _g $.
\end{proposition}
\begin{proof}
	 Let $H= \nabla^2 f(x^*)$, $J= \rmD g(x^*) =\prod_{j=1}^n (\Id - \alpha e_{n-j+1}e_{n-j+1}^{T} H)$, and $y_0$ be the eigenvector corresponding to the smallest eigenvalue of $H$.
	
	 We shall prove that $\norm{J^t y_0}_2 \ge c(1+\epsilon)^t $  for some $\epsilon$ which depends on  $\eta,c>0$, but not on $t$. Applying Gelfand's theorem,
	$$
	\rho(J) = \lim_{t\to \infty} \norm{J^t }^{1/t} \ge (1+\epsilon),
	$$
	and thus $J$ has an eigenvalue of magnitude greater than $1+\epsilon$.

	We fix some arbitrary iteration $t$ and let $y_t = J^t x_0$. We will first show that there exists an $\epsilon>0$  so that
	\begin{equation}\label{eq:inequalitygeom}
	y_{t+1}^{T}Hy_{t+1} \leq (1+\epsilon) y_{t}^{T}Hy_{t},
	\end{equation}
	for all $t\in \mathbb{N}$. Let $z_1 = y_t$ and $z_{i+1} = (\Id - \alpha e_ie_i^{T}H)z_i  = z_i - \alpha (e_i^{T}Hz_i)e_i $, so that $y_{t+1} = Jy_t = z_{n+1}$. We see that the sequence $z_{i+1} ^T H z_{i+1}$ is decreasing (non-increasing),
	\begin{align}
	z_{i+1}^{T} &Hz_{i+1} = [z_i^{T} - \alpha (e_i^{T}Hz_i)e_i^{T}] H [z_i - \alpha (e_i^{T}Hz_i)e_i]\nonumber
	\\& = z_i^{T}Hz_i -\alpha (z_i^{T}He_i)(e_i^{T}Hz_i) - \alpha (e_i^{T}Hz_i)e_i^{T}Hz_i + \alpha^2 (e_i^{T}Hz_i)^2 e_i^{T}He_i \nonumber
	\\& = z_i^{T}Hz_i - \alpha(z_i^{T}He_i)^2(2 - \alpha e_i^{T}He_i) \nonumber
	\\& < z_i^{T}Hz_i - \alpha(z_i^{T}He_i)^2
	\label{eq:use-claim-here},
	\end{align}
	where the last inequality uses that $\alpha < \frac{1}{L_{\max}}$.
	
Next we use the claim to show a sufficient decrease by lower bounding $(z_i ^T H e_i )^2$.
	
	\begin{claim}\label{claim:ineq} Let $y_t$ be in the range of $H$. There exists a $j \in [d]$ so that $\alpha |e_j^{T}Hz_j| \geq \delta \norm{z_j}_2$ for some global constant $\delta>0$ that depends on $H, d$.
	\end{claim}
	\begin{proof}
		We assume that $\alpha |e_j^{T}Hz_j| < \delta \norm{z_j}_2$ for all $j \in \{1,...,n\}$, for some $\delta$ to be chosen later. For $j=1$, it holds that $\norm{y_t-z_2}_2 = \norm{z_1-z_2}_2 = \alpha |e_1^{T}Hz_1| < \delta \norm{z_1}_2 < 2\delta \norm{y_t}_2$ and $\norm{z_2}_2 < (1+2\delta)\norm{y_t}_2$. Suppose for $j \geq 2$ that $\norm{y_t-z_{j}}_2 < 2(j-1)\delta \norm{y_t}_2$ and thus $\norm{z_j}_2 < [1+2(j-1)\delta] \norm{y_t}_2.$ Using induction and triangle inequality we get
		\begin{align*}
		\norm{y_t - z_{j+1}}_2 &\leq \norm{y_t - z_{j}}_2+ \norm{z_{j} -z_{j+1}}_2\\& = 2(j-1)\delta \norm{y_t}_2 + \alpha |e_j^{T}Hz_j| \\&< 2(j-1)\delta \norm{y_t}_2 + \delta \norm{z_j}_2 \\&< 2(j-1)\delta \norm{y_t}_2+ \delta[1+2(j-1)\delta] \norm{y_t}_2
		\\& \leq  2j\delta \norm{y_t}_2,
		\end{align*}
		where we assume $\delta < \frac{1}{2d}$ so that $2(j-1)\delta<1$ for all $j \in [d]$. Using the above calculation,
		\begin{align*}
		\alpha |e_i^{T}Hy_t| &< \alpha|e_i^{T}Hz_i|+\alpha|e_i^{T}H(y_t - z_i)| \\&< \delta \norm{z_i}_2+ \alpha\norm{He_i}_2\norm{y_t - z_i}_2
		\\& < \delta \big(1 +2(i-1)\delta\big) \norm{y_t}_2 + \alpha  \norm{He_i}_2 \big(2(i-1)\delta\big) \norm{y_t}_2\\
		&\le \delta\big(  1+2d \delta + 2d \alpha L\big)\norm{y_t}.
		\end{align*}
		Thus $\alpha \norm{Hy_t}_2 < \sqrt{d}\delta\big(  1+2d \delta + 2d \alpha L\big) \norm{y_t}_2$, and
		\begin{align*}
		\sigma_{\min^+} (H) \norm{y_t}_2 \le \norm{Hy_t}_2 < \frac{\sqrt{d}}{\alpha}\delta\big(  1+2d \delta + 2d \alpha L\big) \norm{y_t}_2,
		\end{align*}
		where $\sigma_{\min^+}$ is the smallest non-zero singular value of $H$.
		Thus by choosing $\delta$ small enough such that  $$
		\sigma_{min^+} (H) \ge \frac{\sqrt{d}}{\alpha}\delta\big(  1+2d \delta + 2d\alpha L\big),
		$$
		we have obtained  a contradiction.
	\end{proof}
	
	Decompose $y_t= y_{\cN} + y_{\cR}$ into the orthogonal components defined by the nullspace  $\nullspace(H)$ and range space $\range(H)$. Notice that $J$ acts as the identity on $\nullspace(H)$, so
	\begin{align*}
	y_{t+1}&=J y_t = y_{\cN} +Jy_{\cR}.
	\end{align*}
	Define an auxiliary sequence $\bar y_{t+1} = J \bar y_t$, and $\bar y_t = y_{\cR}$. Similarly, $\bar z_1 = y_{\cR}$,  $\bar z_{i+1} = (\Id- \alpha e_i e_i ^T  H ) \bar z_i$, and $\bar z_{d+1} = \bar y_{t+1}$. It follows that
	
	\begin{align*}
	y_{t+1} ^T H y_{t+1} & = (y_{\cN} + Jy_{\cR})^T H (y_{\cN} + Jy_{\cR})\\
	&= (J y_{\cR})^T H ( J y_{\cR})\\
	&= \bar z_{d+1} ^T H \bar z_{d+1}\\
	&\le \bar z_{j+1}^T H \bar z_{j+1} \tag{non-increasing property in Equation \eqref{eq:use-claim-here}} \\
	&< \bar z_j ^T H \bar z_j - \alpha ( \bar z_j ^T H e_j )^2 \tag{using Equation \eqref{eq:use-claim-here}} \\
	&< \bar z_j ^T H \bar z_j - \frac{\delta^2}{\alpha} \norm{\bar z_j}_2 ^2  \tag{using Claim \ref{claim:ineq}}\\
	&\le \bar z_j ^T H \bar z_j + \frac{\delta^2}{\alpha L} \bar z_j ^T H \bar z_j \tag{since $L \norm{\bar z_j}^2 \ge \bar z_j ^T H \bar z_j$ }\\
	&= \left(1+\frac{\delta^2}{\alpha L}\right) \bar z_j ^T H \bar z_j\\
	&\le \left(1+\frac{\delta^2}{\alpha L}\right) \bar z_1 ^T H \bar z_1 \tag{non-increasing property}\\
	&= \left(1+\frac{\delta^2}{\alpha L}\right)\bar y_t ^T H \bar y_t\\
	&= \left(1+\frac{\delta^2}{\alpha L}\right)y_t ^T H y_t.
	\end{align*}
Let $\epsilon =\frac{\delta^2}{\alpha L}$. By inducting, and noting that $y_0 ^T H y_0 = -\lambda$,
	\begin{align*}
	y_t ^T H y_t &\le (1+\epsilon)^t y_0 ^T H y_0\\
	&=-\lambda (1+\epsilon)^t.
	\end{align*}
	Using $-\lambda \norm{y_t}_2 ^2 \le y_t ^T H y_t$,
	\begin{align*}
	- \lambda \norm{y_t}_2 ^2 &\le -\lambda (1+\epsilon)^t\\
	\norm{y_t}_2^2 &\ge (1+\epsilon)^t\\
	\norm{J^t y_0}_2 &\ge (1+ \epsilon)^{t/2}\\
	&\ge \left( 1+ \frac{\epsilon}{4}\right)^{t},
	\end{align*}
	where the last inequality uses that $\epsilon \le \frac12$. By Gelfand's theorem, we have established
	\begin{align*}
	\rho(J) \ge 1+\frac{\epsilon}{4},
	\end{align*}
	and thus $J$ has an eigenvalue of magnitude greater than one. Thus $x^* \in \cA^* _g$.
\end{proof}

By combining Propositions \ref{prop:cd-unstable}, \ref{prop:cd-det}, and Corollary \ref{cor:main-corr}, we have the following:
\begin{corollary}[Coordinate Descent]
	Let $g$ be the coordinate descent algorithm as defined in Equation \eqref{eq:g-cd}. Under Assumption \ref{assumption:lipschitz-cd} and $\alpha< \frac{1}{L_{\max}}$,  the stable set of the strict saddle points has measure zero, meaning $\mu(W_g)=0$.
\end{corollary}

\begin{remark} In the worst-case, $L_{\max } =L$, but in many instances $L_{\max} \ll L$, so coordinate descent can use more aggressive step-sizes. The step-size choice $\alpha < \frac{1}{L_{\max}} $ is standard for coordinate-descent methods \cite{richtarik2011iteration}.
\end{remark}

\subsection{Block Coordinate Descent}

The results of this section are a strict generalization of the previous section, but we present the coordinate descent case separately, since the proofs are considerably shorter.

We partition the set $[d]=\{1, 2, \dots, d\}$ to $b$ blocks $\{S_1, S_2, \dots, S_b\}$ such that $[d]= \cup_{i} S_i$. For ease of notation, we define $S_0=\emptyset$.

\begin{algorithm}
	\caption{Block Coordinate Descent} \label{alg:bcd}
	\textbf{Input:} Function $f: \mathbb{R}^n \to \R$, step size $\alpha$, initial point $x_0 \in \mathbb{R}^n$ \\
	\textbf{For} $k = 0,1,\dots$,\\
	\textbf{For block} $i = 1,\dots,b$\\
	\textbf{For index} $j$ \textbf{in block} $i$\\
	\Indp $x^j_{k+1} \leftarrow x^j_{k} -\alpha \frac{\partial f(y^{S_{i-1}}_{k})}{\partial x_j}$, where
	\begin{eqnarray}
	y^{S_0}_k = x_k \textrm{ and }y^{S_i}_k = (x^{S_1}_{k+1},\ldots,x^{S_i}_{k+1},x^{S_{i+1}}_{k},\ldots,x^{S_b}_k)
	\end{eqnarray}
\end{algorithm}

We define $g_i(x)$ to be the block coordinate descent update of block $i$ in Algorithm \ref{alg:bcd}. Block coordinate gradient descent is a dynamical system
\begin{align}
x_{k+1} = g(x_k) = g_b \circ g_{b-1} \circ \ldots \circ g_1,\label{eq:g-bcd}
\end{align}
 where $g_i (x) = x - \alpha \sum_{j\in S_i} e_j^T \nabla f(x)$. We define the matrix $P_S= \sum_{i \in S} e_{i}e_{i}^{T}$, i.e., the projector onto the entries in $S$.

\begin{lemma}
The differential is
	\begin{equation}\label{eq:block_jacobian}
	\rmD g(x_k) = \prod_{i=1}^b (\Id - \alpha P_{S_{b-i+1}} \nabla^2 f(y_k^{S_{b-i}})).
		\end{equation}
\end{lemma}
\begin{proof}
	This is an application of the chain rule. The differential of the composition of two functions $f \circ h$ is just $\rmD f(h(x)) \cdot \rmD h(x)$. By repeatedly applying this and observing that $\rmD g_i (x) = \Id - \alpha P_{S_i}\nabla^2 f(x)$, we obtain the result.
\end{proof}

\begin{assumption}
 	\label{assumption:lipschitz-bcd}
Let $f \in \cC^2$, and $\nabla^2 f(x)_{S}$ be the submatrix of $\nabla^2 f(x)$ by extracting the rows and columns indexed by $S$.  Let $L_b =\max_{i \in [b]} \norm{\nabla^2 f(x)_{S_i}}_2$	
	\end{assumption}

\begin{proposition}Under Assumption \ref{assumption:lipschitz-bcd} and $\alpha<\frac{1}{L_b}$, then $\det( \rmD g(x) ) \neq 0$.
\label{prop:bcd-det}
\end{proposition}

\begin{proof}
	It suffices to prove that every term of the product \ref{eq:block_jacobian} is an invertible matrix. Every matrix of the form $\Id - \alpha P_{S_i} \nabla^2 f(x)$  has $n-|S_i|$ eigenvalues equal to one and the rest of its eigenvalues correspond
	to eigenvalues of $\Id_{S_i} - \alpha \nabla^2 f(x) _{S_i, S_i}$. Since $\alpha< \frac{1}{L_b}$, then the eigenvalues of  $\Id_{S_i} - \alpha \nabla^2 f(x) _{S_i}$ are all greater than zero. Thus each $\Id - \alpha P_{S_i} \nabla^2 f(x)$ is invertible, and $\rmD g$ is also invertible.

\end{proof}

\begin{proposition}[Stability at fixed points] Let $x^*$ be a strict saddle point of $f$. The Jacobian of the update rule of block coordinate descent computed at point $x^*$ has an eigenvalue of modulus greater than one.
	\label{prop:bcd-unstable}
\end{proposition}
\begin{proof} Let $H= \nabla^2 f(x^*)$,  $J =\rmD g(x^*) = \prod_{i=1}^b (\Id - \alpha P_{S_{b-i+1}} H)$, and $y_0$ be an eigenvector of the Hessian at $x^*$.
	
	 We shall prove that $\norm{J^t y_0}_2 \ge c(1+\eta)^t $. Hence by Gelfand's theorem $J$ must have at least one eigenvalue with magnitude greater than one. The proof technique is very similar to that of the proof of Proposition \ref{prop:cd-unstable}.
	
	We fix some arbitrary iteration $t$ and let $y_t = J^t x_0$. We will first show that there exists an $\epsilon>0$,
	\begin{equation}\label{eq:inequalitygeomblock}
	y_{t+1}^{T} Hy_{t+1} \leq (1+\epsilon) y_{t}^{T}Hy_{t},
	\end{equation}
	for all $t\in \mathbb{N}$.
	Let $z_1 = y_t$ and $z_{i+1} = (\Id - \alpha P_{S_i} H)z_i  = z_i - \alpha \sum_{j \in S_i}(e_j^{T}Hz_i)e_j $, so that $y_{t+1} = Jy_t = z_{b+1}$. We get that
	\begin{align*}
	&z_{i+1}^{T}Hz_{i+1}  = \big(z_i^{T} - 2 \alpha \sum_{j \in S_i}(e_j^{T}Hz_i)e_j^{T}\big) H \big( z_i - \alpha  \sum_{j \in S_i}(e_j^{T}Hz_i)e_j\big)
	\\& = z_i^{T}Hz_i -2\alpha \sum_{j \in S_i}(e_j^{T} H z_i)^2 + \alpha^2 \big(\sum_{j \in S_i}(e_j^{T}Hz_i)e_j\big)^{T} H \big(\sum_{j \in S_i}(e_j^{T}Hz_i)e_j \big)
	\\& < z_i^{T}Hz_i - 2\alpha \sum_{j \in S_i}(e_j ^{T} H z_i)^2 + \alpha^2 L_{b}  \norm{\sum_{j \in S_i}(e_j^{T}Hz_i)e_j}_2^2 \tag{ using $\norm{H_{S_i}}_2 \le L_b$}
	\\& = z_i^{T}Hz_i - \alpha (2 -\alpha L_b  )\norm{\sum_{j \in S_i}(e_j^{T}Hz_i)e_j}_2^2 \\
	&\le z_i^{T}Hz_i  - \alpha \norm{\sum_{j \in S_i}(e_j^{T}Hz_i)e_j}_2^2 \tag{ using $\alpha L_b <1$}.
	\end{align*}
	Thus $z_{i}^T H z_i $ is a decreasing (non-increasing) sequence.

We shall prove that there exists an $i \in [b]$ so that $z_{i+1}^{T}Hz_{i+1} \leq (1+\delta) z_{i}^{T}Hz_{i}$ for some global constant $\delta$ to be chosen later.
	
	\begin{claim}\label{claim:ineqblock} Let $y_t$ be in the range of $H$. There exists an $i \in [b]$ so that $\alpha \sum_{j \in S_i} \left|e_j^{T}Hz_i\right| \geq \delta \norm{z_i}_2$ for some  $\delta>0$.
	\end{claim}
	\begin{proof}
		We assume that $\alpha \sum_{j \in S_i}\left|e_j^{T}Hz_i \right| < \delta \norm{z_i}_2$ for all $i \in [b]$. For $i=1$, it holds that $\norm{y_t-z_2}_2 = \norm{z_1-z_2}_2 = \alpha |\sum_{j \in S_1}e_j^{T}Hz_1| < \alpha \sum_{j \in S_1} |e_j^{T}Hz_1| < \delta \norm{z_1}_2 < 2\delta \norm{y_t}_2$ and $\norm{z_2}_2 < (1+2\delta)\norm{y_t}_2$. Suppose for $i \geq 2$ that $\norm{y_t-z_{i}}_2 < 2(i-1)\delta \norm{y_t}_2$ and thus $\norm{z_i}_2 < [1+2(i-1)\delta] \norm{y_t}_2.$ Using induction and triangle inequality we obtain
		\begin{align*}
		\norm{y_t - z_{i+1}}_2 &\leq \norm{y_t - z_{i}}_2+ \norm{z_{i} -z_{i+1}}_2\\& = 2(i-1)\delta \norm{y_t}_2 + \alpha \left|\sum _{j \in S_i}e_j^{T}Hz_i\right| \\&\leq 2(i-1)\delta \norm{y_t}_2 + \alpha \sum _{j \in S_i}\left|e_j^{T}Hz_i\right| \\&< 2(i-1)\delta \norm{y_t}_2 + \delta \norm{z_i}_2 \\&< 2(i-1)\delta \norm{y_t}_2+ \delta[1+2(i-1)\delta] \norm{y_t}_2
		\\& \leq  2i\delta \norm{y_t}_2,
		\end{align*}
		where we assume $\delta < \frac{1}{2b}$ so that $2(i-1)\delta<1$ for all $i \in [b]$. Using the above,
		\begin{align*}
		\alpha \sum_{j \in S_i} \left| e_j^{T}Hy_t\right| &< \alpha \sum_{j \in S_i}\left |e_j^{T}Hz_i \right|+\alpha \sum_{j \in S_i}\left|e_j^{T}H(y_t - z_i)\right| \\&< \delta \norm{z_i}_2+ \alpha\left(\sum_{j \in S_i}\norm{He_j}_2 \right)\norm{y_t - z_i}_2
		\\& < \delta [1+2(i-1)\delta]\norm{y_t}_2 + \alpha[ 2(i-1)\delta] \norm{y_t}_2 \left(\sum_{j \in S_i}\norm{He_j}_2 \right).
		\end{align*}
		Since $\norm{He_i}_2 < \sigma_{\max}(H) \leq L$, we get that $\alpha \sum_{j \in S_i} \norm{He_j}_2 < |S_i| \leq d$ and we conclude
		\begin{equation}\label{eq:boundblock}
		\alpha \sum_{j \in S_i}\left| e_j^{T}Hy_t\right| < 2d^2 \delta \norm{y_t}_2.
		\end{equation}
		Finally, using Inequality \ref{eq:boundblock} it follows that $\alpha \norm{Hy_t}_2 < 2d^2 \delta \sqrt{d}  \norm{y_t}_2$. Let $w \in \textrm{Im}(H)$ be a vector that is orthogonal to $\textrm{null}(H)$ (since $H$ is symmetric). Then it holds that $\norm{Hw}_2 \geq \sigma_{\min^+}(H) \norm{w}_2$ where $\sigma_{\min^+}(H)$ denotes the smallest positive singular value of $H$ (greater than zero). Assume that $y_t \in \textrm{Im}(H)$ and we get $\norm{Hy_t}_2 < \frac{2d^2 \delta \sqrt{d}}{\alpha} \norm{y_t}_2$. However, $\norm{Hy_t}_2 \geq \sigma_{\min^+}(H) \norm{y_t}_2$ thus by choosing $\frac{2d^2 \sqrt{d}\delta}{\alpha} < \sigma_{\min^+}(H)$ we reach a contradiction. The appropriate choice of $\delta$ is any positive constant in $(0, \frac{\alpha\sigma_{\min^+}}{2d^2\sqrt{d}})$ (since $1/2d \geq \frac{1}{2d^2\sqrt{d}}$).
	\end{proof}
	To finish the proof of the lemma, suppose that Claim \ref{claim:ineqblock} applies. Then by Cauchy-Schwarz, there exists an index $i$ such that
	\begin{align*}
	z_{i+1}^{T}Hz_{i+1} &< z_i^{T}Hz_i - \alpha \sum_{j \in S_i}(z_i^{T}He_j)^2 \\
	&< z_i^{T}Hz_i - \frac{\alpha}{d} \left(\sum_{j \in S_i} \left|z_i^{T}He_j\right|\right)^2 < z_i ^{T}Hz_i - \frac{\delta^2}{d\alpha} \norm{z_i}^2_2.
	\end{align*}
	However, $w^{T}Hw \geq \lambda_{\min}(H)\norm{w}^2_2 \geq - L \norm{w}^2_2$, hence we get that
	\begin{equation}\label{eq:ineqproof}
	z_{i+1}^{T}Hz_{i+1} < \left(1+ \frac{\delta^2}{\alpha L d}\right) z_i ^{T}Hz_i .
	\end{equation}
	By choosing $\epsilon=\frac{\delta^2}{\alpha L d}$ we showed that $y_{t+1}^{T}Hy_{t+1} \leq (1+\epsilon)y_t^{T}Hy_t$ as long as $y_t$ is in the range of $H$.
	
	 Assume that $y_t = y_{\cN} + y_{\cR}$. It is easy to see $y_t^{T}Hy_t = y_{\cR}^{T}Hy_{\cR}$ and also $y_{t+1} = Jy_t = y_{\cN} + Jy_{\cR}$, hence $y_{t+1}^{T}Hy_{t+1} =(Jy_{\cR})^{T} H (Jy_{\cR})$. Therefore from Inequality \ref{eq:ineqproof} proved above, if the starting vector is $y_{\cR}$, which Claim \ref{claim:ineqblock} applies too,  then $(Jy_{\cR})^{T}HJy_{\cR} \leq (1+\epsilon)y_{\cR }^{T}H y_{\cR} = (1+\epsilon)y_t^{T}Hy_t$.
	
	To sum up, we showed that $y_t^{T}Hy_t \leq (1+\epsilon)^t y_0^{T} H y_0$ and since $y_0$ is an eigenvector of $H$ (of norm one) with corresponding negative eigenvalue $\lambda$, it follows that $y_t^{T}Hy_t \leq \lambda (1+\epsilon)^t.$
	Finally using $y_t^{T}Hy_t \geq \lambda_{\min}(H)\norm{y_t}^2_2$, we get $\norm{y_t}_2 \geq (1+\epsilon)^{t/2} \frac{\lambda}{\lambda_{\min}(H)}$. Observe that $\frac{\lambda}{\lambda_{\min}(H)}>0$ is a positive constant, $(1+\epsilon)^{t/2} \geq (1+\epsilon/4)^t$ (since $\epsilon \leq 1/2$) and the proof follows (the parameters as claimed in the beginning will be $c = \frac{\lambda}{\lambda_{\min}(H)}$ and $\eta = \epsilon/4$).	
\end{proof}

By combining Propositions \ref{prop:bcd-unstable}, \ref{prop:bcd-det}, and Corollary \ref{cor:main-corr}, we have the following:
\begin{corollary}[Block Coordinate Descent]
	Let $g$ be the block coordinate descent algorithm as defined in Equation \eqref{eq:g-bcd}. Under Assumption \ref{assumption:lipschitz-bcd} and $\alpha< \frac{1}{L_{b}}$,  the stable set of the strict saddle points has measure zero, meaning $\mu(W_g)=0$.
\end{corollary}

\begin{remark}
In the worst-case, $L_{b } =L$, but in many instances $L_{b} \ll L$, so block coordinate descent can use more aggressive step-sizes. The step-size choice $\alpha < \frac{1}{L_{b}} $ is standard for block coordinate descent methods \cite{richtarik2011iteration}.
\end{remark}

\subsection{Manifold Gradient Descent}
\label{sec:manifold-gd}
Let $\cX$ be a submanifold of $\reals^D$, and $\cT(x)$ be the tangent space of $\cX$ at $x$. $P_{\cX}$ and $P_{\cT(x)}$ be the orthogonal projector onto $\cX$ and $\cT(x)$ respectively. Let $\bar f$ be a smooth extension of $f$ to  $\reals^D$, and $f= \bar f|_{\cM}$.
The manifold gradient descent algorithm is:
\begin{align}
x_{k+1}= g(x)  \triangleq P_{\cM} ( x_k - \alpha P_{\cT(x_k)} \nabla \bar f(x_k)). \label{eq:g-manifold-gd}
\end{align}
Recall that the Riemannian gradient $\nabla_{R} f(x) = P_{\cT(x)} \nabla \bar f(x)$, so the above iteration is precisely manifold gradient descent with $P_{\cM}$ as retraction.

%

\begin{proposition}
	\label{prop:manifold-gd-saddle-unstable}
At a strict saddle point $x^*$,  $\rmD g (x^*)$ has an eigenvalue of magnitude larger than $1$.
\end{proposition}
\begin{proof}
Since $x^*$ is a strict saddle, the Riemannian Hessian $\nabla ^2 _{R} f (x^*)$ has a negative eigenvalue $\lambda_v$ and eigenvector $v$, and $P_{\cT(x^*)} \nabla \bar f(x^*) =0$.

Using \cite[Lemma 4]{lewis2008alternating} , $\rmD P_{\cM} (x) = P_{\cT(x)}$ for $x \in \cM$,
\begin{align*}
\rmD g(x^*) v  =P_{\cT(x^*)} v- \alpha P_{\cT(x^*)} \rmD( P_{\cT} \nabla \bar f)(x^*) v.
\end{align*}
Using \cite[Equation 4]{absil2013extrinsic}, $P_{\cT(x)}\rmD (P_{\cT} \nabla \bar f)(x) v = \nabla_{R} ^2 f(x) v$, so
\begin{align*}
\rmD g(x^*) v= v- \lambda_v v .
\end{align*}
Thus $v$ is an eigenvector of $\rmD g(x^*)$ with eigenvalue $1-\lambda_v >1$.
\end{proof}

\begin{proposition}
For a compact submanifold $\cM$, there is a strictly positive $\alpha$ such that $\det( \rmD g) \neq 0$. \label{prop:manifold-gd-det}
\end{proposition}
\begin{proof}
Since $\cM$ is a compact smooth manifold, $P_{\cM}$ is unique and smooth in a neighborhood of radius $r$ of the manifold \cite{absil2012projection}. Letting $\alpha < \frac{r}{\max_{x \in \cM} \norm{\nabla f(x)}}$,  $P_{\cM} (x- \alpha P_{\cT(x)} \nabla f(x))$ and its derivatives exist.

We wish to show that $\rmD g(x) = \rmD P_{\cM} ( x- \alpha P_{\cT(x)} \nabla f(x)) ( \Id- \alpha \rmD(P_{\cT} \nabla f) (x) )$ is invertible. Define $h_x(\alpha) = \det\left( \rmD P_{\cM} ( x- \alpha P_{\cT(x)} \nabla f(x)) ( \Id- \alpha \rmD(P_{\cT} \nabla f) (x) )\right)$.  Using $\rmD P_{\cM} (x) = P_{\cT(x)}$ \cite{absil2012projection}, $h_x(0) =1 $. Since $B:=\max_{x \in \cM, \alpha<\frac{r}{\max_{x \in \cM} \norm{\nabla f(x)}}}| \frac{d h_x }{d \alpha} (\alpha) | <\infty$, we see that for \[\alpha < C_{\cX,f}\triangleq \min \left(\frac{r}{\max_{x \in \cM} \norm{\nabla f(x)}}, \frac1B \right)\] that $h_x (\alpha)$ is positive, so $\rmD g$ is invertible.
\end{proof}

\begin{corollary}
	Let $g$ be the manifold gradient descent algorithm of Equation \eqref{eq:g-manifold-gd}, and $\cX$ be a compact sub-manifold of $\reals^D$. Then there is a $C_{\cX,f} >0$, that only depends on the properties of $\cX$ and $f$, such that for any step-size $\alpha<C_{\cX, f} $  ,  the stable set of the strict saddle points has measure zero, meaning $\mu(W_g) =0$.
	\end{corollary}
\begin{proof}
	The proof is a straightforward application of the previous two Propositions and Corollary \ref{cor:main-corr}. Proposition \ref{prop:manifold-gd-saddle-unstable} shows that $\cX^* \subset \cA^* _g$, and Proposition \ref{prop:manifold-gd-det} shows that $\det( \rmD g (x)) \neq 0$. By applying Corollary \ref{cor:main-corr}, we conclude that $\mu( \{x_0: \lim_k g^k(x_0) \in \cX^*\}) =0$. The parameter $C_{\cX,f}$ is specified in the proof of Proposition \ref{prop:manifold-gd-det}.
\end{proof}

\subsection{Mirror Descent}
In this section, we consider the mirror descent algorithm. Let $\cD$ be a convex open subset of $\reals^D$, and $\cX = \cD \cap \cM$ for some affine space $\cM$. Given a mirror map $\Phi$, we define the mirror descent algorithm in Algorithm \ref{alg:md}.
\begin{algorithm}[!h]
	\caption{Mirror Descent} \label{alg:md}
	\textbf{Input:} Function $f: \cX \to \R$, step size $\alpha$, and initial point $x_0$. \\
	\textbf{For} $k = 0,1,\dots$,\\
	\Indp Update $$x_{k+1} \leftarrow h\left(\nabla \Phi(x_k) -\alpha\nabla f(x_k)\right),$$ where $ h(x) \triangleq \arg\max_{z \in \cX}  z^T x - \Phi(z)$.
\end{algorithm}

Before we continue, we provide an example of a commonly used instantiation of mirror descent known as the Multiplicative Weights algorithm.
\begin{example}[Probability Simplex]\label{example:prob-simplex}
Define the mirror map $\Phi(x) = \sum x_i \log x_i $, with $\cD$ being the positive orthant $\reals^{D} _{>0}$, and affine space $\cM = \{x: \sum_i x_i =1\}$. The domain is $\cX = \cD \cap \cM$ which is the interior of probability simplex. The mirror descent algorithm corresponds to the update :
\begin{align*}
 x_i \leftarrow \frac{x_i \exp( - \alpha \frac{\partial f}{\partial x_i} (x) )  }{\sum_{j} x_j \exp( - \alpha \frac{\partial f}{\partial x_j} (x) )}.
\end{align*}
	\end{example}

We define $\bar \cX$ to be the closure of $ \cX$, and  $\partial \cX = \bar \cX \backslash \cX $ to be the relative boundary of $\cX$. Due to the affine constraint, $\cX$ may not be full-dimensional, so we define the appropriate notions of gradient and Hessian.  Let $\cT$ be the tangent space of $\cM$. The Riemannian gradient is $\nabla_{R} \Phi(x) = P_{\cT} \nabla \Phi(x) $. Similarly the Riemannian Hessian is $\nabla^2_{R} \Phi (x) = P_{\cT} \nabla^2 \Phi(x) P_{\cT}$ and is a linear mapping from $\cT \to \cT$. Finally, the mirror descent mapping is defined as
\begin{align}
g(x) =h \circ F (x) \label{eq:g-md},
 \end{align} with $F(x) = \nabla \Phi(x) - \alpha \nabla f(x)$ and $h(x) =\arg\max_{z \in \cX} z^T x - \Phi(z)$.

\begin{assumption}[Mirror Map]
\label{assumption:mirror-map}
We say that $\Phi$ is a \emph{mirror} map if it satisfies the following properties:
\begin{enumerate}
	\item$\Phi: \cD \to \reals$ is $\cC^2$ and strictly convex.
	\item  The gradient of $\Phi$ is surjective onto $\reals^D$, that is $\nabla \Phi(\cD)=\reals^D$.
	\item $\nabla_{R} \Phi$ diverges on the relative boundary of $\cX$, that is $\lim_{x \to \partial \cX} \norm{\nabla_{R} \Phi (x)} =\infty$. Furthermore, the negative gradient points inwards, that is for $x_0 \in \partial \cX$,   $\lim_{x \to x_0 } - \frac{\nabla_{R} \Phi(x)}{\norm{\nabla_{R} \Phi(x)}} \in T_{\bar \cX} (x_0)$, where $T_{\cK}$ denotes the tangent cone of the set $\cK$.
	 \end{enumerate}
\end{assumption}

 \begin{assumption}[Strong convexity of $\Phi$ and Lipschitz Gradient]
 	Let $\Id_{\cT}$ be the identity mapping on $\cT$. We assume that
 	\begin{enumerate}
 		\item $\Phi$ is $\mu$-strongly convex, meaning $\nabla^2 _{R} \Phi(x) \succeq \mu \Id_{\cT}$.
 		\item $f$ has $L$-Lipschitz gradient, meaning $\nabla^2 _{R} f(x) \preceq L \Id_{\cT}$.
 	\end{enumerate}
 	\label{assumption:strcvx-lip}
 \end{assumption}
 \begin{remark}
 	In the simplex example of Example \ref{example:prob-simplex}, the strong convexity parameter satisfies $\mu \ge 1$.
 \end{remark}

 We first express the mapping $g$ as a composition of simple mappings.
\begin{lemma}
	Assume that $\Phi$ is a $\mu$-strongly convex mirror map. The mirror descent algorithm can be equivalently expressed as
$g(x) = (\nabla_{R} \Phi)^{-1} \circ P_{\cT}  \circ F (x)$, and $(\nabla_{R} \Phi)^{-1}:\cT \to \cX $ is a local diffeomorphism.
\label{lem:md}
	\end{lemma}
\begin{proof}
	Recall that $h(w) =\arg\max_{z \in \cX} z^Tw - \Phi(z)$. Let  $H(z) =  -z^Tw +\Phi(z)$. By strong convexity, $H$ attains an unique minimizer in  $\bar \cX$.

We first show that the minimizer $z^* \notin \partial \cX$. For contradiction, let us assume $z^* \in \partial \cX$. By the first-order optimality conditions, $\lim_{x \to z^*} -\frac{\nabla_{R} H(x)}{\norm{\nabla_{R} H(x)}} \in N_{\bar \cD}(z^*) $, where $N_{\bar \cD}$ is the normal cone of the closure of $\cD$. Using \cite{rockafellar2009variational}[Theorem 6.9 and 6.42] and $N_{\cM} (x)  = \cT^\perp$, $N_{\bar \cX}  (x) \supset N_{\bar \cD} (x) + \cT^\perp$, where $+$ denotes Minkowski sum. Thus $\lim_{x \to z^*} -\frac{\nabla_{R} H(x)}{\norm{\nabla_{R} H(x)}} \in N_{\bar \cX}(z^*) $.

By assumption, $\lim_{x \to z^*}-\frac{ \nabla_R \Phi(x)}{\norm{\nabla_R \Phi(x)}} \in \cT_{\bar \cX} (z^*) $.
Since the tangent cone and normal cone are polar cones,
\begin{align*}
0 &\ge \lim_{x\to z^*}\left(  \frac{\nabla_R H(x)}{\norm{\nabla_R H(x)}}\right) ^T \left(\frac{\nabla_R \Phi(x)}{\norm{\nabla_R \Phi(x)}}\right) \\
&= \lim_{x \to z^*} \frac{1}{\norm{\nabla_R \Phi(x)} \norm{\nabla_R H(x)}} \big(-w^T \nabla_R \Phi(x) +\norm{\nabla_R \Phi(x)}^2\big) \\
&\ge\lim_{x \to z^*} \frac{1}{\norm{\nabla_R \Phi(x)} \norm{\nabla _RH(x)}}\left( - \norm{w} \norm{\nabla_R \Phi(x)}+\norm{\nabla_R \Phi(x)}^2\right)\\
&=1,
\end{align*}
where the inequality uses Cauchy-Schwartz , and the  last equality uses that $\lim_{x\to z^*} \norm{\nabla_R \Phi(x)}=\infty$. This gives a contradiction, so we must have that $z^* \in \cX$.
	
	By first-order optimality conditions, $	\nabla_{R} \Phi(h(w)) = P_{\cT} w $, and thus
	\begin{align*}
	h(w) = (\nabla_{R} \Phi)^{-1} \circ P_{\cT} (w).
	\end{align*}
As a shorthand, let $\Psi  =(\nabla_{R} \Phi)^{-1}$. By existence and uniqueness of the maximizer, $\Psi $ is a single-valued function from $\cT \to \cX$.
Thus $g= h \circ F = \Psi \circ P_{\cT} \circ F$.


Next we verify that $\Psi$ is a local diffeomorphism. By the inverse function theorem, $\rmD \Psi  (\nabla_{R} \Phi(x)) = (\nabla^2 _{R} \Phi (x) )^{-1}$. Taking determinants, we see that $\det\left( \rmD \Psi (\nabla_{R} \Phi(x))\right) = \det \left((\nabla^2 _{R} \Phi (x) )^{-1}\right)>0$, using strict convexity of $\Phi$. Thus $\Psi $ is a local diffeomorphism from $\cT\to \cX$.
\end{proof}

\begin{proposition}
	Under Assumptions \ref{assumption:mirror-map}, \ref{assumption:strcvx-lip} and $\alpha< \frac{\mu}{L}$, then
	\begin{enumerate}
	\item $\det(\rmD g(x)) \neq 0$.
	
	\item Every strict saddle point of $x^*$ is an unstable fixed point of mirror descent, meaning $x^* \in \cA^* _g$.
	\end{enumerate}
\label{prop:md}
	\end{proposition}
\begin{proof}
	Again adopt the shorthand $\Psi = (\nabla_{R} \Phi)^{-1} $. 	Using Lemma \ref{lem:md}, $g= \Psi \circ \left(P_{\cT} \circ F\right) $, and $\Psi $ is a local diffeomorphism. To show $\det( \rmD g) \neq 0 $, it suffices to show that $P_{\cT} \circ F$ is a local diffeomorphism, or equivalently verify that $ \rmD ( P_{\cT} \circ F) (x): \cT \to \cT$ is an invertible linear transformation for every $x \in \cX$.
	\begin{align*}
	\rmD ( P_{\cT} \circ F) (x) &= P_{\cT} \nabla^2 \Phi (x) P_{\cT} - \alpha P_{\cT} \nabla^2 f (x) P_{\cT} \\
	&= \nabla_{R} ^2 \Phi(x) - \alpha \nabla_{R} ^2  f(x).
	\end{align*}
	By the Lipschitz assumption, $\nabla^2_{R} f(x) \preceq  L \Id_{\cT}$ and $\alpha < \frac{\mu}{L}$. By the strong convexity assumption of $\Phi$, $\nabla^2 _{R} \Phi(x) \succeq  \mu \Id_{\cT}$. Thus
	\begin{align*}
	\alpha \nabla_{R}^2 f(x) \preceq \mu  \Id_{\cT} \preceq \nabla^2 _{R} \Phi(x).
	\end{align*}
	Using the calculation above,  $\nabla_{R} ^2 \Phi(x) - \alpha \nabla_{R} ^2  f(x) \succ 0$ and is invertible.  This completes our proof of the first part.
	
	Let $x^* \in \cX $ be a strict saddle point. First we verify that it is a fixed point of $g$. Using that $x^*$ is a critical point,
	\begin{align*}
	g(x^*) &=\Psi  ( P_{\cT} \nabla\Phi(x^*) -P_{\cT}\nabla f(x^*)) \\
	& = \Psi ( P_{\cT}\nabla  \Phi(x^*) )\\
	 &= \Psi ( \nabla_{R} \Phi (x^*))\\
	 &=x^*.
	\end{align*}
	Next we verify that $\rmD g(x^*)$ has an eigenvalue of magnitude greater than one. Using the chain rule and then inverse function theorem,
	\begin{align*}
	\rmD g(x^*) &= \rmD \Psi (P_{\cT} \circ F(x^*))  (\nabla^2_{R} \Phi (x^*) - \alpha \nabla^2 _{R} f(x^*)) \tag{chain rule}\\
	&=\nabla^2 _{R} \Phi( x^*) ^{-1} (\nabla^2_{R} \Phi (x^*) - \alpha \nabla^2 _{R} f(x^*) \tag{inverse function theorem and $\Psi\circ P_{\cT} \circ F(x^*) =x^*$ }\\
	&= \Id_{\cT} - \alpha \nabla^2 _{R} \Phi( x^*) ^{-1} \nabla^2 _{R} f(x^*).
	\end{align*}
Define $A= \nabla_{R} ^2 \Phi(x^*)$ and $ H= \nabla^2 _{R} f(x^*)$ . By similarity transformation under $A^{1/2}$,
\begin{align*}
A^{1/2} Dg(x^*)A^{-1/2} = \Id_{\cT} - \alpha A^{-1/2} H A^{-1/2} ,
\end{align*}
which is  a symmetric linear operator. Define $\bar v = A^{1/2} v $, where $v$ is an eigenvector of $\nabla^2 _{R} f(x^*)$ corresponding to a strictly negative eigenvalue $\lambda$, then $\bar v^T A^{-1/2} H A^{-1/2} \bar v <0$, so $\lambda_{\min} ( A^{-1/2} H A^{-1/2}) <0$. Thus $ 1-\alpha \lambda_{\min} ( A^{-1/2} H A^{-1/2})$ is an eigenvalue of $ \Id_{\cT} - \alpha A^{-1/2} H A^{-1/2}$ that is greater than one. Since similarity transformations preserve eigenvalues, $Dg(x^*)$ also has an eigenvalue greater than one, and so $x ^* \in \cA^* _g$.
\end{proof}

By combining Proposition \ref{prop:md} with Corollary \ref{cor:main-corr}, we have the following:
\begin{corollary}
	Let $g$ be the mirror descent algorithm defined in Equation \eqref{eq:g-md}.	Under Assumptions \ref{assumption:mirror-map}, \ref{assumption:strcvx-lip}, and $\alpha< \frac{\mu}{L}$, then the stable set of the strict saddles in $\cX$ is measure zero, meaning $\mu(W_g ) =0$.
	\end{corollary}
\begin{remark}
This corollary does not guarantee that the stable set of saddles on $\partial \cX$ is measure zero. For example in Multiplicative Weights algorithm, there are fixed points on $\partial \cX$ (\eg\ all the vectors $x$ with support size 1).
\end{remark}

\section{Conclusion}
\label{sec:conclusion}

We have shown that first-order methods with random initialization and appropriate constant step-size do not converge to a saddle point. Our results apply to gradient descent, proximal point algorithm, coordinate descent, block coordinate descent, manifold gradient descent and mirror descent. The key common insight in analyzing all these optimization methods is to treat these algorithms as dynamical systems. Every strict saddle point is shown to be locally unstable for these first-order methods and applications of the center-stable manifold theorem suffice to characterize the local behavior. As long as the mapping induced by the optimization method is sufficiently well behaved, \eg\ local diffeomorphism, these local arguments can be extended to the whole domain. Proving the instability of saddle points as well as the smoothness and invertibility of the corresponding maps depends upon careful instantiations of these generic arguments (\eg\, choice of step-size) on a case-by-case basis. The global instability of saddle points for first-order methods is many times informally invoked without careful discussion about the necessary technical conditions needed to formalize these arguments. We hope that this work
 will help ground these arguments on a unified formal foundation.
 We  end this paper with a brief discussion of some open directions:

 {\bf Step-size.} It is not clear if the step size restrictions are necessary to avoid saddle points (\eg\ $\alpha<1/L$ for gradient descent; see \cite{panageas2017gradient} in which examples are provided where $\alpha < 2/L$ is necessary for gradient descent).  Most of the constructions where the gradient method converges to saddle points require fragile initial conditions as discussed in Section~\ref{sec:intuition}.  It remains a possibility that adaptive choice of step-size by Wolfe Line Search or backtracking, may still avoid saddle points provided the initial point is chosen at random.

 {\bf Strict saddles.} It is also important to understand how stringent the strict saddle assumption is. Will a perturbation of a function always satisfy the strict saddle property? \cite{adler2009random} provide very general sufficient conditions for a random function to be Morse, meaning the eigenvalues at critical points are non-zero, which implies the strict saddle condition. These conditions rely on checking that the density of $\nabla^2 f(x)$ has full support conditioned on the event that $\nabla f(x)=0$. This can be explicitly verified for functions $f$ that arise from learning problems. Similar arguments for applications that arise in game theory are developed in \cite{kleinberg2009multiplicative}.

However, we note that there are very difficult unconstrained optimization problems where the strict saddle condition fails.  Perhaps the simplest is optimization of quartic polynomials.  Indeed, checking if zero is a local minimizer of the quartic
\[
	f(x) = \sum_{i,j=1}^n q_{ij} x_i^2 x_j^2
\]
is equivalent to checking whether the matrix $Q=[q_{ij}]$ is co-positive, a co-NP complete problem.  For this $f$, the Hessian at $x=0$ is zero, so $x=0$ is a second-order KKT point, but not necessarily a local minimizer. By the change of variables $z_i = x_i^2$, we see that checking local minimality in a problem with quadratic objective and non-negative inequality constraints is also co-NP complete.

{\bf Speed of convergence.}  Although gradient descent can take exponential amount of time to escape from saddle points at least for some carefully constructed non-convex functions \cite{du2017gradient},  its stochastic counterparts perform much better \cite{jin2017escape}. It would be interesting to characterize these hard instances to the extent possible and to understand whether they are indeed prevalent in applications of interest (\eg\ deep learning). In the other direction, it would be rather useful to show that all first-order methods can be sped up by switching to carefully chosen stochastic variants.

 {\bf Beyond saddle points.}  Even if saddle points are provably avoided, there can be multiple local minima of widely different objective value. The performance of first-order methods would depend crucially on whether they converge for most initial conditions to nearly optimal global minima.  \cite{panageas2016average} analyze such a game theoretic application and show that indeed the size of the region of attraction of the good local optima dominates that of the bad local optima implying nearly optimal average case performance. Such arguments depend crucially both on the setting as well as on the chosen optimization method and it would be interesting to explore their applicability in other settings.

\bibliographystyle{plain}      
\bibliography{gradient}

\end{document}